\RequirePackage{fix-cm}
\documentclass[twocolumn]{svjour3}

\usepackage{graphicx}
\usepackage{amsmath}
\usepackage{amssymb}
\usepackage[linesnumbered,ruled]{algorithm2e}
\usepackage{color}
\usepackage{url}
\usepackage{cite}

\usepackage{amsthm}
\usepackage{subfigure}

\newtheorem{assumption}{Assumption}
\newtheorem{prob}{Problem}

\newcommand{\WinnerTakesAll}{\textsc{\texttt{WinnerTakesAll}}}
\newcommand{\WinnerTakesAllFirst}{\textsc{\texttt{WinnerTakes}}}
\newcommand{\WinnerTakesAllSecond}{\textsc{\texttt{All}}}
\newcommand{\WinnerTakesAllThird}{\textsc{\texttt{Winner}}}
\newcommand{\WinnerTakesAllFourth}{\textsc{\texttt{TakesAll}}}
\newcommand{\Bottleneck}{\textsc{\texttt{Bottleneck}}}
\newcommand{\BottleneckFirst}{\textsc{\texttt{Bottlene}}}
\newcommand{\BottleneckSecond}{\textsc{\texttt{ck}}}

\newcommand{\eg}{\textit{e.g.}}
\newcommand{\ie}{\textit{i.e.}}
\newcommand{\etal}{\textit{et al}.}

\newcommand{\argmax}{\arg\!\max}

\begin{document}

\title{Distributed Assignment with Limited Communication for Multi-Robot Multi-Target Tracking
\thanks{This material is based upon work supported by the National Science Foundation under Grant No. 1637915.}
}

\titlerunning{Distributed Assignment with Limited Communication}   

\author{Yoonchang Sung \and 
Ashish Kumar Budhiraja \and
Ryan K. Williams \and
Pratap Tokekar
}

\authorrunning{Sung et al.} 

\institute{Y. Sung \at
              \email{yooncs8@vt.edu}           
           \and
           A. K. Budhiraja \at
              \email{ashishkb@vt.edu}
           \and
           R. K. Williams \at
              \email{rywilli1@vt.edu}
           \and   
           P. Tokekar \at
              \email{tokekar@vt.edu}
           \and
           Department of Electrical \& Computer Engineering, Virginia Tech, 1185 Perry Street, Blacksburg, VA 24061, USA
}

\date{Received: date / Accepted: date}

\maketitle

\begin{abstract}
We study the problem of tracking multiple moving targets using a team of mobile robots. Each robot has a set of motion primitives to choose from in order to collectively maximize the number of targets tracked or the total quality of tracking. Our focus is on scenarios where communication is limited and the robots have limited time to share information with their neighbors. As a result, we seek distributed algorithms that can find solutions in a bounded amount of time. We present two algorithms: (1) a greedy algorithm that is guaranteed to find a $2$--approximation to the optimal (centralized) solution but requiring $|R|$ communication rounds in the worst case, where $|R|$ denotes the number of robots; and (2) a \emph{local} algorithm that finds a $\mathcal{O}\left((1+\epsilon)(1+1/h)\right)$--approximation algorithm in $\mathcal{O}(h\log 1/\epsilon)$ communication rounds. Here, $h$ and $\epsilon$ are parameters that allow the user to trade-off the solution quality with communication time. In addition to theoretical results, we present empirical evaluation including comparisons with centralized optimal solutions.
\end{abstract}

\section{Introduction}~\label{sec:int}
We study the problem of assigning robots with limited Field-Of-View (FOV) sensors to track multiple moving targets. Multi-robot multi-target tracking is a well-studied topic in robotics~\cite{parker1997cooperative,parker2002distributed,touzet2000robot,kolling2007cooperative,honig2016dynamic}. We focus on scenarios where the number of robots is large and solving the problem locally rather than centrally is desirable. The robots may have  limited communication range and limited bandwidth. As such, we seek assignment algorithms that rely on local information and only require a limited amount of communication with neighboring robots. 

Constraints on communication impose challenges for robot coordination as global information may not always be available to all the robots within the network. As a result, it may not be always possible to design algorithms that operate on local information while still ensuring global optimality. Recently, Gharesifard and Smith~\cite{gharesifard2017distributed} studied how limited information due to the communication graph topology affects the global performance. Their analysis applies for the case when the robots are allowed only one round of communication with their neighbors. If the robots are allowed multiple rounds of communication, they can propagate the information across the network. Given sufficient rounds of communication, all robots will have access to global information, and therefore can essentially solve the centralized version of the problem. In this paper, we investigate the relationship between the number of communication rounds allowed for the robots and the performance guarantees. We focus on the problem of distributed multi-robot, multi-target assignment for our investigation (Figure~\ref{fig:description}).

\begin{figure}[thpb]
\centering
\includegraphics[scale=0.36]{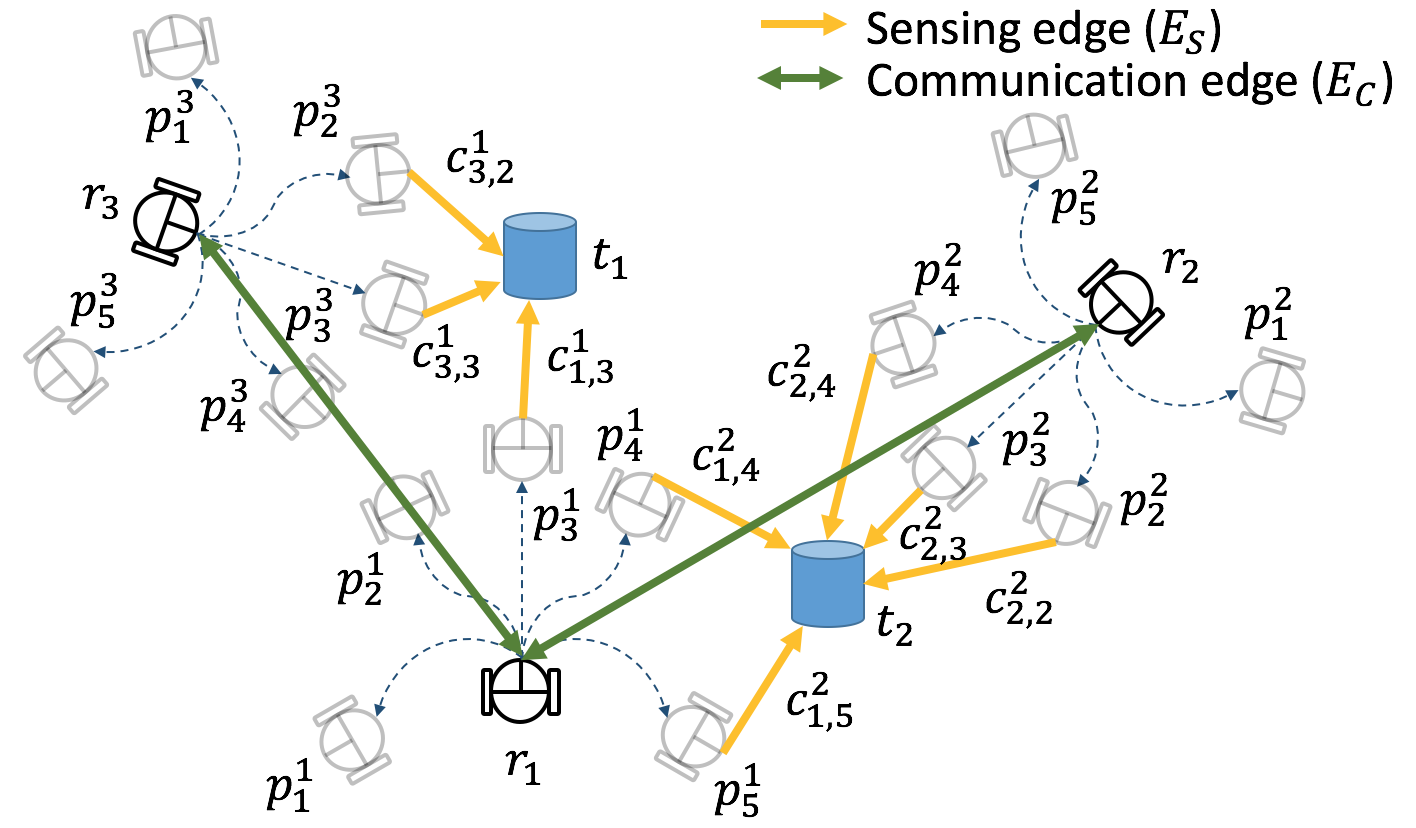}
\caption{Description of multi-robot task allocation for multi-target tracking. In this example, three robots ($\textbf{r}_1, \textbf{r}_2, \textbf{r}_3$) are tracking two moving targets ($\textbf{t}_1, \textbf{t}_2$). Each robot has five motion primitives ($\textbf{p}_m^i$) to choose from at each time step. $c$ represents the cost of observing a target from a motion primitive.
}
\label{fig:description}
\end{figure}

We assume that each robot has a number of motion primitives to choose from. A motion primitive is a local trajectory obtained by applying a sequence of actions~\cite{howard2014model}. A motion primitive can track a target if the target is in the FOV of the robot. The set of targets tracked by different motion primitives may be different. The assignment of targets to robots is therefore coupled with the selection of motion primitives for each robot. Our goal is to assign motion primitives to the robots so as to track the most number of targets or maximize the quality of tracking. We term this as the distributed Simultaneous Action and Target Assignment (SATA) problem. 

This problem can be viewed as the dual of the set cover problem, known as the maximum (weighted) cover \cite{suomela2013survey}. Every motion primitive covers some subset of the targets. Therefore, we would like to pick motion primitives that maximize the number (or weight) of covered targets. However, we have the additional constraint that only one motion primitive per robot can be chosen at each step. This implies that the relationship between a robot and the corresponding motion primitives turns out to be a packing problem~\cite{suomela2013survey} where only one motion primitive can be ``packed'' per robot. The combination of the two aforementioned problems is called a Mixed Packing and Covering Problem (MPCP)~\cite{young2001sequential}. 

We study two versions of the problem. The first version can be formulated as a (sub)modular maximization problem subject to a partition matroid constraint~\cite{nemhauser1978analysis}. A sequential greedy algorithm, where the robots take turns to greedily choose motion primitives, is known to yield a $2$--approximation for this problem~\cite{tokekar2014multi}. We evaluate the empirical performance of this algorithm by comparing it with a centralized (globally optimal) solution. The drawback of the sequential greedy algorithm is that it requires at least as many communication rounds as the number of robots. This may be too slow in practice. Consequently, we study a second version of the problem for which we present a \emph{local} algorithm whose performance degrades gracefully (and provably) as a function of the number of communication rounds.

\begin{figure}[thpb]
\centering
\includegraphics[scale=0.50]{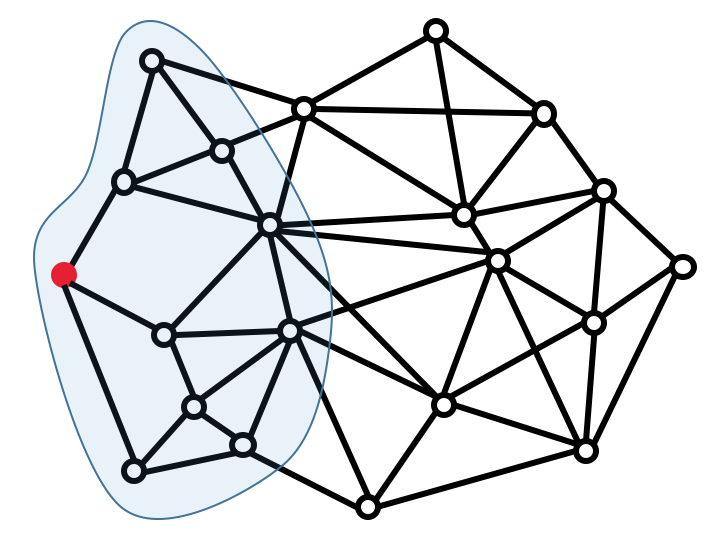}
\caption{Communication graph. The blue shaded region indicates a radius--2 neighborhood of the red solid node. The red solid node may be unaware of the entire communication graph topology. A local algorithm that works for the red solid node only requires local information of nodes in the blue shaded region. The same local algorithm runs on all the nodes and ensures bounded approximation guarantees on the global optimality.}
\label{fig:local}
\end{figure}
A local algorithm~\cite{suomela2013survey} is a constant-time distributed algorithm that is independent of the size of a network. This enables a robot only to depend on local inputs in a fixed-radius neighborhood of robots (Figure~\ref{fig:local}). The robot does not need to know information beyond its local neighborhood, thereby achieving better scalability.

Flor{\'e}en~\etal~\cite{floreen2011local} proposed a local algorithm to solve MPCP using max-min/min-max Linear Programming (LP) in a distributed manner. We show how to leverage this algorithm to solve SATA. This algorithm has a bounded communication complexity unlike typical distributed algorithms. Specifically, the algorithm yields a $\mathcal{O}\left((1+\epsilon)(1+1/h)\right)$ approximation to the globally optimal solution in $\mathcal{O}(h\log{1/\epsilon})$ synchronous communication rounds where $h$ and $\epsilon$ are input parameters.\footnote{An algorithm is called a $\mathcal{O}(x)$ approximation to a maximization problem if it guarantees a solution whose value is at least $\frac{c}{x}$ of the optimal value, where $c$ is some constant.} We verify the theoretical results through empirical evaluation.

The contributions of this paper are as follows:
\begin{enumerate}
\item We present two versions of the SATA problem.
\item We show how to use the greedy algorithm and adapt the local algorithm for solving the two versions of the SATA problem.
\item We perform empirical comparisons of the proposed algorithm with baseline centralized solutions.
\item We demonstrate the applicability of the proposed algorithm through Gazebo simulations. 
\end{enumerate}

A preliminary version of this paper was presented at ICRA 2018~\cite{sung2018distributed}. This expanded paper extends the preliminary version with a more thorough literature survey, additional theoretical analysis, and significantly expanded empirical analysis including a description of how to implement the greedy algorithm in practice.

The rest of the paper is organized as follows. We begin by introducing the related work in Section~\ref{sec:related}. We describe the problem setup in Section~\ref{sec:prob}. Our proposed distributed algorithms are presented in Section~\ref{sec:distributed}. We present results from representative simulations in Section~\ref{sec:sim} before concluding with a discussion of future work in Section~\ref{sec:conc}.

\section{Related Work}~\label{sec:related}

A number of algorithms have been designed to improve multi-robot coordination under limited bandwidth \cite{yan2013survey,li2017high,otte2013any,otte2018dynamic,kassir2016communication} and under communication range constraints~\cite{williams2013constrained,vander2015algorithms,kantaros2015distributed}. This includes algorithms that enforce connectivity constraints~\cite{kantaros2017distributed,williams2015global}, explicitly trigger when to communicate~\cite{dimarogonas2012distributed,zhou2018active,ge2017distributed1} and operate when connectivity is intermittent~\cite{best2018planning,guo2018multirobot}. In this section, we focus on work that is most closely related to the SATA problem and local algorithms.

\subsection{Multi-Robot Target Tracking} \label{subsec:related_tracking}

There have been many studies on cooperative target tracking in both control and robotics communities. We highlight some of the recent related work in this section. For a more comprehensive overview of multi-robot multi-target tracking, see the recent surveys~\cite{khan2016cooperative,robin2016multi}. 

Charrow~\etal~\cite{charrow2014approximate} proposed approximate representations of the belief to design a control policy for multiple robots to track one mobile target. The proposed scheme, however, requires a centralized approach. Yu~\etal~\cite{yu2015cooperative} worked on an auction-based decentralized algorithm for cooperative path planning to track a moving target. Ahmad~\etal~\cite{ahmad2017online} presented a unified method of localizing robots and tracking a target that is scalable with respect to the number of robots. Zhou and Roumeliotis~\cite{zhou2011multirobot} developed an algorithm that finds an optimal trajectory of multiple robots for the active target tracking problem. Capitan~\etal~\cite{capitan2013decentralized} proposed a decentralized cooperative multi-robot algorithm using auctioned partially observable Markov decision processes. The performance of decentralized data fusion under limited communication was successfully shown but theoretical bounds on communication rounds were not covered. Moreover, theoretical properties presented in the above references considered single target tracking, which may not necessarily hold in the case of tracking multiple targets in a distributed fashion.



Pimenta~\etal~\cite{pimenta2009simultaneous} adopted Voronoi partitioning to develop a distributed multi-target tracking algorithm. However, their objective was to cover an environment coupled with multi-target tracking.
Banfi~\etal~\cite{banfi2018integer} addressed the \emph{fairness} issue for cooperative multi-robot multi-target tracking, which is achieving balanced coverage among different targets. One of the problems that we define in Section~\ref{sec:prob} (\ie, Problem~\ref{prob:bottleneck}) has a similar motivation. However, unlike the algorithm in Banfi~\etal~\cite{banfi2018integer}, we are able to give a global performance guarantee. Xu~\etal~\cite{xu2013decentralized} presented a decentralized algorithm that jointly solves the problem of assigning robots to targets and positioning robots using mixed-integer nonlinear programming. While they proved the complexity in terms of computational time and communication (\ie, the amount of data needed to be communicated), the solution quality was only evaluated empirically. Instead, we bound the solution quality as a function of the communication rounds. Furthermore, our formulation takes as input a set of discrete actions (\ie, motion primitives) that the robot must choose from, unlike the previous work.

We study a problem similar to the one termed as Cooperative Multi-robot Observation of Multiple Moving Targets (CMOMMT) proposed by Parker and Emmons~\cite{parker1997cooperative}. The objective in CMOMMT is to maximize the collective time of observing targets. 
Parker~\cite{parker2002distributed} developed a distributed algorithm for CMOMMT that computes a local force vector to find a direction vector for each robot. We empirically compare this algorithm with our proposed one and report the results in Section~\ref{sec:sim}.  
Kolling and Carpin \cite{kolling2007cooperative} studied the behavioral CMOMMT that added a new mode (\ie, help) to the conventional track and search modes of CMOMMT. The help mode asks other robots to track a target if the target escapes from the FOV of some robot. Although our work does not allow mode changes, previous works regarding CMOMMT did not provide theoretical optimality guarantees and did not explicitly consider scenarios where the communication bandwidth is limited.
Refer to Section IV(C) of Reference~\cite{khan2016cooperative} for a more detailed summary of CMOMMT.

In our prior work~\cite{tokekar2014multi}, we addressed the problem of selecting trajectories for robots that can track the maximum number of targets using a team of robots. However, no bound on the number of communication rounds was presented, possibly resulting in all-to-all communication in the worst case. Instead, in this work, we introduce a new version of the problem and also explicitly bound the amount of communication required for target assignment.

\subsection{Multi-Robot Task Assignment} \label{subsec:related_assignment}

Multi-robot task assignment can be formulated as a discrete combinatorial optimization problem. The work by Gerkey and Matari\'c~\cite{gerkey2004formal} and the more recent work by Korsah~\etal~\cite{korsah2013comprehensive} contain detailed survey of this problem. There exists distributed algorithms with provable guarantees for different versions of this problem~\cite{choi2009consensus,luo2015distributed,liu2011assessing}. 
There also exists various multi-robot deployment strategies for task assignment under  communication constraints. These constraints include limited available information~\cite{kanakia2016modeling}, limited communication flows~\cite{le2012adaptive}, and connectivity requirement~\cite{kantaros2016global}. See the survey papers~\cite{zavlanos2011graph,ge2017distributed2} on these results. 
Ny~\etal~\cite{le2012adaptive} studied a formulation with a similar communication constraint as ours. However, their formulation assumed that the robots know which targets to track. In this paper, we tackle the challenge of simultaneously assigning robots to targets by choosing motion primitives with limited communication bandwidth which might degrade task performance when there are unreliable communication links and communication delays.


Turpin~\etal~\cite{turpin2014capt} proposed a distributed algorithm that assigns robots to goal locations while generating collision-free trajectories.
Morgan~\etal~\cite{morgan2016swarm} solved the assignment problem by using distributed auctions and generating collision-free trajectories by using sequential convex programming. 
Bandyopadhyay~\etal~\cite{bandyopadhyay2017probabilistic} adopted the Eulerian framework for both swarm formation control and assignment.
However, these works may not be suitable for target tracking applications as the targets were assumed to be static.
For more survey results about SATA, see the work by Chung~\etal~\cite{chung2018survey}.
Recently, Otte~\etal~\cite{otte2017multi} investigated the effect of communication quality on auction-based multi-robot task assignment. None of the above works, however, analyzed the effect of communication rounds on the solution quality, as is the focus of our work.

\subsection{Local Algorithms} \label{subsec:local_alg}

A local algorithm~\cite{angluin1980local,linial1992locality,naor1995can} is a distributed algorithm that is guaranteed to achieve desired objective in a finite (typically, fixed) amount of time. The typical approach is to find approximate solutions with provable (and global) performance guarantees while ensuring a bound on the communication complexity that is independent of the number of vertices in the graph. Local algorithms have been proposed for a number of graph-theoretic problems. These include, graph matching~\cite{hanckowiak2001distributed}, vertex cover~\cite{aastrand2009local,aastrand2010fast}, dominating set~\cite{lenzen2010minimum}, and set cover~\cite{kuhn2006price}. Suomela~\cite{suomela2013survey} gives a broad survey of local algorithms. We build on this work and adapt a local algorithm for solving SATA.



\section{Problem Description} \label{sec:prob}
Consider a scenario where multiple robots are tracking multiple mobile targets. Robots can observe targets within their FOV and predict the future states of targets. Based on predicted target states, robots decide where to move (\ie, by selecting a motion primitive) in order to keep track of targets. By discretizing time, the problem becomes one of combinatorial optimizations --- choose the next position of robots based on the predicted position of the targets. Thus, we solve the SATA problem at each time step.

We define sets, $R$ and $T$, to denote the collection of robot and target labels respectively: $R=\{1,...,i,...,|R|\}$ for robot labels and $T=\{1,...,j,...,|T|\}$ for target labels.
Let $r$ and $t$ denote the set of robot states and \emph{predicted} target states, respectively. In this paper, states are given by the positions of the robots and the targets in 2- or 3-dimensional space. However, the algorithms presented in this paper can be used for more complex states (\eg, 6 degree-of-freedom pose).
Here, $r(k)=\{\textbf{r}_{1}(k),$ $...,\textbf{r}_{i}(k),...,$ $\textbf{r}_{|R|}(k)\}$ denotes the state of the robots at time $k$. $t(k)=\{\textbf{t}_{1}(k),...,\textbf{t}_{j}(k),$ $...,\textbf{t}_{|T|}(k)\}$ denotes the state of the targets at the next time step (\ie, at time $k+1$) predicted at time $k$.
We assume that the targets can be uniquely detected and multiple robots know if they are observing the same target. Therefore, no data association is required. Each robot independently obtains the predicted states, $t(k)$, by fusing its own noisy sensor measurements using, for example, a Kalman filter. 

We define the labels of available motion primitives for the $i$-th robot as $P^i=\{1,...,m,...,|P^i|\}.$
These labels correspond to a set of motion primitive states of the $i$-th robot at time $k$ given by: $p^i(k)=\{\textbf{p}_{1}^i(k),...,\textbf{p}_{m}^i(k),$ $...,\textbf{p}_{{|P^i|}}^i(k)\}$. Note that the term \emph{motion primitives} in this paper represents the future state of a robot at the next time step (\ie, at time $k+1$) computed at time $k$.
We compute a set of the motion primitives a priori by discretizing the continuous control input space. This can be done by various methods such as uniform random sampling or biased sampling based on predicted target states. However, once a set of the motion primitives is obtained, the rest of the proposed algorithms (in Section~\ref{sec:distributed}) remain the same.

We define $\mathcal{RS}(\textbf{p}_{m}^i(k))$ to be the set of targets that can be observed by the $m$-th motion primitive of $i$-th robot at time $k$. Specifically, the $j$-th target is said to be observable by the $m$-th motion primitive of a robot $i$, iff $\textbf{t}_{j}(k)\in \mathcal{RS}(\textbf{p}_{m}^i(k))$. It should be noted that only targets that were observed by robot $i$ at time $k-1$ are candidates to be considered for time $k$ because unobserved targets at time $k-1$ cannot be predicted by the robot $i$. Note also that since $\mathcal{RS}$ is a set function, we can model complex FOV and sensing range constraints that are not necessarily restricted to 2D.

We make the following assumptions.\footnote{After these assumptions, we omit the time index (\ie, $k$) for notational convenience.
}
\begin{assumption}
\textbf{(Communication Range).}
If two robots have a motion primitive that can observe the same target, then these robots can communicate with each other. This implies if there exists a target $j$ such that $\textbf{t}_{j}(k)\in \mathcal{RS}(\textbf{p}_{m}^i(k))$ and $\textbf{t}_{j}(k)\in \mathcal{RS}(\textbf{p}_{m}^l(k))$, then $i$-th and $l$-th robots can communicate with each other.
\label{assum:range}
\end{assumption}

\begin{assumption}
\textbf{(Synchronous Communication).}
All the robots have synchronous clocks leading to synchronous rounds of communication. 
\label{assum:sync}
\end{assumption}


From Assumption~\ref{assum:range}, neighboring robots can share their local information with each other when they observe the same targets. For example, robots can use techniques such as covariance intersection~\cite{niehsen2002information} to merge their individual predictions of the target's state into a joint prediction $T$. This can be achieved in one round of communication when each robot simply broadcasts its own estimate of all the targets within its FOV. Note that a robot does not need to know the prediction for all the targets but only the ones that are within the FOV of one of its motion primitives. In this sense, a communication graph $\mathcal{G}_C=(R,E_C)$ can be created from a sensing graph $\mathcal{G}_S=(P\cup{T},E_S)$ at each time, where $E_C$ and $E_S$ denote edges among robots and edges between targets and motion primitives, respectively.


As shown in Figure~\ref{fig:description}, each robot is able to compute feasible motion primitives of its own and detect multiple unique targets within the FOV. Then, the objective of the proposed problem is to choose one of the motion primitives for each robot, yielding either the best quality of tracking or the maximum number of targets tracked by the robots, depending on the application. One possible quality of tracking can be measured by the summation of all distances between selected primitives and the observed targets.

Let $x_{m}^{i}$ be the binary variable which represents the $i$-th robot selecting the $m$-th motion primitive. That is, $x_{m}^{i}=1$ if a motion primitive $m$ is selected by a robot $i$ 
and $0$ otherwise.\footnote{If all $x_{m}^{i}=0$ for a robot $i$, then it can choose any motion primitives since the objective value will remain the same.} Since each robot can choose only one motion primitive, we have:
\begin{equation}
\begin{split}
\sum_{m\in P^i}{x_{m}^{i}}\leq{1}\ \ &\forall i\in R.
\end{split}
\label{eqn:variable_x}
\end{equation}


Our objective is to find $x_{m}^i$. We propose two following problems.

\begin{prob}[\Bottleneck{}]
\label{prob:bottleneck}
The objective is to select primitives such that we maximize the minimum tracking quality:
\begin{equation}
\argmax_{x_{m}^{i}}\quad\min_{j\in T}\left(\sum_{i\in R}\sum_{m\in P^i}c_{i,m}^jx_{m}^{i}\right),
\label{eqn:bottleneck}
\end{equation}
subject to the constraints in Equation (\ref{eqn:variable_x}). Here, $c_{i,m}^j$ denotes weights on sensing edges $E_S$ between $m$-th motion primitive of $i$-th robot and $j$-th target.
\end{prob}

Here, $c_{i,m}^j$ can represent the tracking quality given by, for example, the inverse of the distance between $m$-th motion primitive of $i$-th robot and $j$-th target. Alternatively, $c_{i,m}$ can be binary (1 when the $m$-th motion primitive of robot $i$ sees target $j$ and $0$ otherwise) making the objective function 
equal to maximizing
the minimum number of targets tracked.

We term this as the \Bottleneck{} version of SATA. 
In the \Bottleneck{} version, multiple robots may be assigned to the same target. We also define a \WinnerTakesAllThird{}- \WinnerTakesAllFourth{} variant of SATA where only one robot is assigned to a target.


We define additional binary decision variable, $y_{i}^{j}$. $y_{i}^{j}$ represents the $i$-th robot assigned to track the $j$-th target. We have, $y_{i}^{j}=1$ if $i$-th robot is assigned to $j$-th target and $0$ otherwise. 

Since we restrict only one robot to be assigned to the target (unlike \Bottleneck{}), we have:
\begin{equation}
\begin{split}
\sum_{i\in R}{y_{i}^{j}}\leq{1}\ \ &\forall j\in T.
\end{split}
\label{eqn:variable_y}
\end{equation}

\begin{prob}[\WinnerTakesAll{}]
\label{prob:winnertakesall}
The objective is to maximize the total quality of tracking given by,
\begin{equation}
\argmax_{x_{m}^{i},y_{i}^{j}}\sum_{j\in T}\left(\sum_{i\in R}y_{i}^{j}\left(\sum_{m\in P^i}c_{i,m}^jx_{m}^{i}\right)\right),
\label{eqn:objective}
\end{equation} 
subject to the constraints in Equations (\ref{eqn:variable_x}) and (\ref{eqn:variable_y}).
\end{prob}





Both versions of the SATA problem are NP-Hard~\cite{vazirani2001approximation}. The \WinnerTakesAll~version can be optimally solved using a Quadratic Mixed Integer Linear Programming (QMILP) in the centralized setting.\footnote{Note that Problem~\ref{prob:winnertakesall} can also be converted into a simpler Mixed Integer Linear Programming (MILP) by linearizing the product of the binary variables in Equation (\ref{eqn:objective}), which is not covered in this paper.} Our main contributions are to show how to solve both problems in a distributed manner: an LP-relaxation of the \Bottleneck{} variant using a local algorithm; and the \WinnerTakesAll{} variant using a greedy algorithm. The following theorems summarize the main contributions of our work.

\begin{theorem}~\label{theorem:proposed_local}
Let $\bigtriangleup_R\geq2$ be the maximum number of motion primitives per robot and $\bigtriangleup_T\geq2$ be the maximum number of motion primitives that can see a target. There exists a local algorithm that finds an $\bigtriangleup_R(1+\epsilon)(1+1/h)(1-1/\bigtriangleup_T)$ approximation in $\mathcal{O}(h\log{1/\epsilon})$ synchronous communication rounds for the LP-relaxation of the \Bottleneck~version of SATA problem, where $h$ and $\epsilon>0$ are parameters.
\end{theorem}
The proof follows directly from the existence of the local algorithm described in the next section.
We show how the local algorithm for MPCP can be modified to solve SATA by means of a linear relaxation.

\begin{theorem}~\label{theorem:proposed_greedy}
There exists a $2$--approximation greedy algorithm for the \WinnerTakesAll~version of the SATA problem for any $\epsilon>0$ in polynomial time.
\end{theorem}
This directly follows from the fact that the problem is a modular maximization problem subject to a partition matroid constraint~\cite{nemhauser1978analysis}. The algorithms are described in the next section.


\section{Distributed Algorithms} \label{sec:distributed}
We begin by describing the local algorithm that solves the \Bottleneck{} version of SATA.

\subsection{Local Algorithm} \label{subsec:local}


In this section, we show how to solve the \Bottleneck~version of the SATA problem using a local algorithm. We adapt the local algorithm for solving max-min LPs given by Flor\'een~\etal~\cite{floreen2011local} to solve the SATA problem in a distributed manner. 

Consider the tripartite, weighted, and undirected graph, $\mathcal{G}=(R\cup{P}\cup{T},E)$ shown in Figure~\ref{fig:general}. Each edge $e\in{E}$ is either $e=(\textbf{r}_i,\textbf{p}_m^i)$ with weight 1 or $e=(\textbf{t}_j,\textbf{p}_m^i)$ with weight $c_{i,m}^{j}$. The maximum degree among robot nodes $\textbf{r}_i\in{r}$ is denoted by $\bigtriangleup_R$ and among target nodes $\textbf{t}_j\in{t}$ is $\bigtriangleup_T$. Each motion primitive $\textbf{p}_m^i\in{p^i}$ is associated with a variable $x_{m}^{i}$.
The upper two layers of $\mathcal{G}$ in Figure~\ref{fig:general} are related with a packing problem (Equation (\ref{eqn:objective})). The lower two layers are related with the covering problem. 

\begin{lemma}~\label{lemma:equiv_bottleneck}
The \Bottleneck~version (Equation (\ref{eqn:bottleneck})) can be rewritten as a linear relaxation of ILP:
\begin{equation}
\begin{split}
\mbox{maximize}\ \ \ &w \\
\mbox{subject to}\ \ \ &\sum_{m\in{P^i}}x_{m}^{i}\leq{1}\ \ \forall{i\in{R}} \\
&\sum_{i\in{R}}\sum_{m\in{P^i}}c_{i,m}^{j}x_{m}^{i}\geq{w}\ \ \forall{j\in{T}} \\
&\ \ \ \ \ \ \ \ \ \ x_{m}^{i}\geq{0}\ \ \forall{m\in{P^i}}.
\end{split}
\label{eqn:mpcp}
\end{equation}
\end{lemma}
The proof is given in Appendix~\ref{append:equiv_bottleneck}.

\begin{figure}[thpb]
\centering
\includegraphics[scale=0.38]{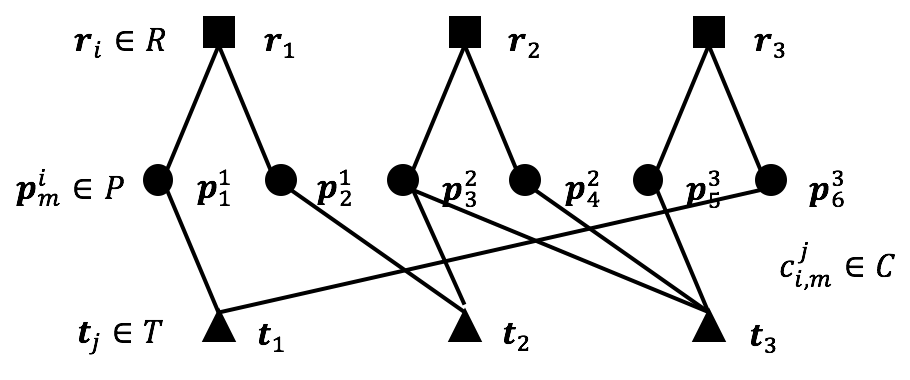}
\caption{One instance of a graph for MPCP when there are three robot nodes, six motion primitive nodes and three target nodes.}
\label{fig:general}
\end{figure}

Flor\'een~\etal~\cite{floreen2011local} presented a local algorithm to solve MPCP in Equation (\ref{eqn:mpcp}) in a distributed fashion. They presented both positive and negative results for MPCP. We show how to adopt this algorithm for solving the \Bottleneck~version of SATA.

\begin{figure}[thpb]
\centering
\includegraphics[scale=0.48]{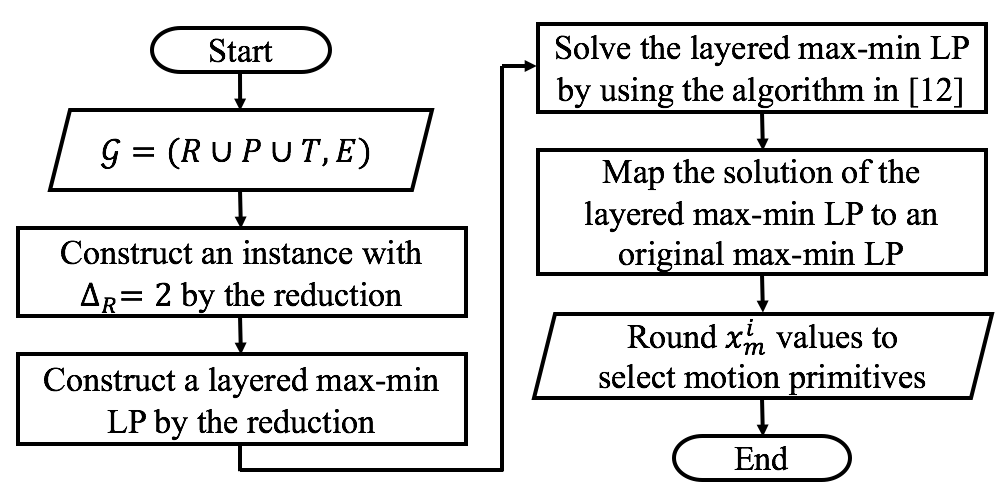}
\caption{Flowchart of the proposed local algorithm.}
\label{fig:flowchart}
\end{figure}

An overview of our algorithm is given in Figure~\ref{fig:flowchart}. We describe the main steps in the following.

\subsubsection{Local Algorithm from Reference~\cite{floreen2011local}}\label{subsubsec:local}
The local algorithm in Reference~\cite{floreen2011local} requires~$\bigtriangleup_R=2$. However, they also present a simple local technique to split nodes in the original graph with $\bigtriangleup_R > 2$ into multiple nodes making $\bigtriangleup_R = 2$. Then, a \emph{layered} max-min LP is constructed with $h$ layers, as shown in Figure~\ref{fig:layered}. $h$ is a user-defined parameter that allows to trade-off computational time with optimality. If the number of layers is set to $h$, then it means that a robot can communicate with another robot that is no more than $h$ communication edges (\ie, hops) away. The layered graph breaks the symmetry that inherently exists in the original graph. This layered mechanism is specifically designed for solving MPCP and is covered in depth in Section 4 of Reference~\cite{floreen2011local}. We omit the details in this paper due to limited space and redirect the readers to Section 4 of Reference~\cite{floreen2011local} for the construction of the layered graph.

\begin{figure}[thpb]
\centering
\includegraphics[scale=0.50]{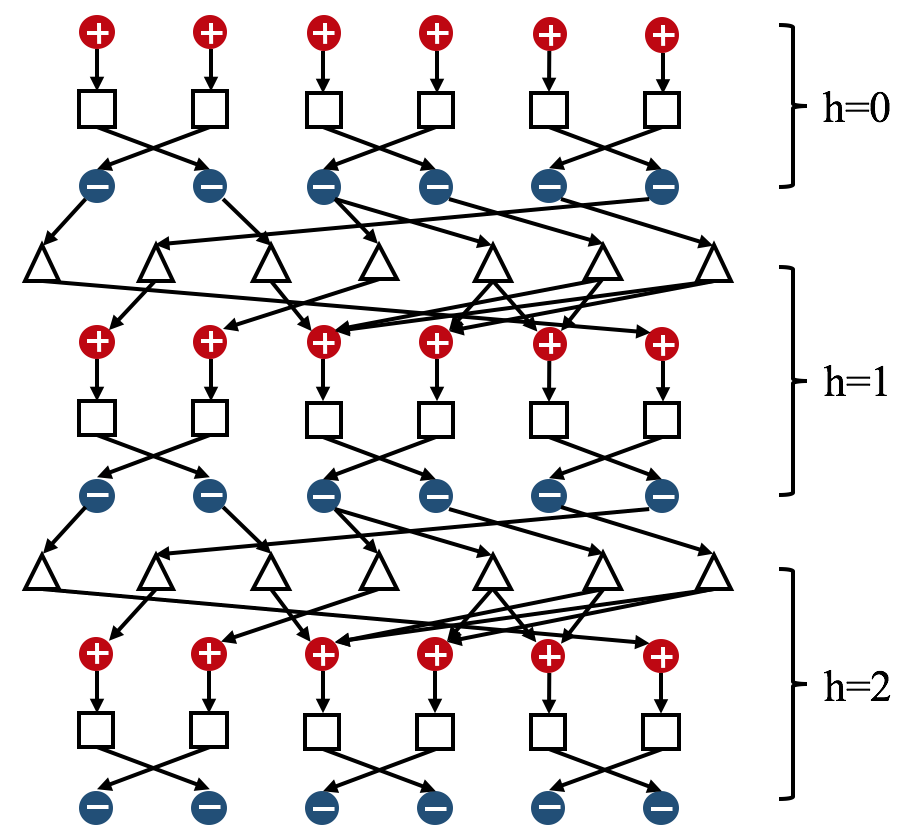}
\caption{Graph of the layered max-min LP with $h=2$ that is obtained from the original graph of Figure~\ref{fig:general} after applying the local algorithm. The details for constructing a layered graph are given in Section 4 of Reference~\cite{floreen2011local}. Each motion primitive $\textbf{p}_m^i\in{p^i}$ is colored either red or blue to break the symmetry of the original graph. Squares, circles, and triangles represent robot nodes, motion primitive nodes, and target nodes, respectively, corresponding to Figure~\ref{fig:general}.
}
\label{fig:layered}
\end{figure}

They proposed a recursive algorithm to compute a solution of the layered max-min LP. The solution for the original max-min LP can be obtained by mapping from the solution of the layered one. The obtained solution corresponds to values of $x_m^i$. They proved that the resulting algorithm gives a constant-factor approximation ratio.
\begin{theorem}
There exists local approximation algorithms for max-min and min-max LPs with the approximation ratio $\bigtriangleup_R(1+\epsilon)(1+1/h)(1-1/\bigtriangleup_T)$ for any $\bigtriangleup_R\geq2$, $\bigtriangleup_T\geq2$, and $\epsilon>0$, where $h$ denotes the number of layers.
\end{theorem}
\begin{proof}
Please refer to Corollary 4.7 from Reference~\cite{floreen2011local} for a proof.
\end{proof}
Note that each node in the layered graph carries out its local computation (details of the local computation for solving SATA are included in Reference~\cite{floreen2011local}). Each node also receives and sends information from and to neighbors at each synchronous communication round. Constructing the layered graph is done in a local fashion without requiring any single robot to know the entire graph.

\subsubsection{Realization of Local Algorithm for SATA} \label{subsubsec:realize}

To apply the local algorithm of Section~\ref{subsubsec:local} to a distributed SATA problem, each node and edge in a layered graph must be realized at each time step (\ie, generating a graph shown in Figure~\ref{fig:layered} which becomes the input to the local algorithm~\cite{floreen2011local}). In our case, the only computational units are the robots. Nodes that correspond to motion primitives, $\textbf{p}_m^i\in{p^i}$, can be realized by the corresponding robot $\textbf{r}_i\in{r}$. Moreover, nodes corresponding to the targets must also be realized by robots. A target $j$ is realized by a robot $i$ satisfying $\textbf{t}_{j}\in \mathcal{RS}(\textbf{p}_{m}^i)$. If there are multiple robots whose motion primitives can sense the target (by Assumption~\ref{assum:range}), they can arbitrarily decide which amongst them realizes the target nodes in a constant number of communication rounds. 


After applying the local algorithm of Section~\ref{subsubsec:local} to robots, each robot obtains $x_{m}^{i}$ on corresponding $\textbf{p}_m^i$ at each time. However, due to the LP relaxation, $x_{m}^{i}$ will not necessarily be binary, as in Equation (\ref{eqn:variable_x}). For each robot we set the highest $x_{m}^{i}$ equal to one and all others as zero. We shortly show that the resulting solution after rounding is still close to optimal in practice. Furthermore, increasing the parameter $h$ finds solutions that are close to binary. 

The following pseudo-code explains the overall scheme of each robot for a distributed SATA. We solve the SATA problem at each time step. 


\begin{algorithm} \label{alg:local}
    \SetKwInOut{Input}{Input}
    \SetKwInOut{Output}{Output}
    \For{$\textbf{r}_{i}(k)\in r(k)$}
    {
      $p^i(k)\leftarrow$ComputeMotionPrimitives($\textbf{r}_{i}(k)$).
      
      Find targets that can be sensed by $m$-th motion primitive of $i$-th robot ($\textbf{p}_{m}^i(k)$).
      
      Construct a $h$-hop communication graph.
      
      Apply local algorithm~\cite{floreen2011local}.
      
      $\hat{x}_{m}^{i}\leftarrow$ Rounding\big($x_{m}^{i}$\big).
      
      $\textbf{p}_m^{i*}(k)\leftarrow$ Motion primitive with $\hat{x}_{m}^{i} = 1$.
      
      ApplyAction\big($\textbf{p}_m^{i*}(k)$\big).
      
      $k\leftarrow k+1$.
      
    }
    \caption{Local algorithm}
\end{algorithm}

\subsubsection{Advantages of the Local Algorithm} \label{subsubsec:observation}

It is possible that there are some robots that are isolated from the others. That is, the communication graph or the layered graph may be disconnected. However, each component of the graph can run the local algorithm independently without affecting the solution quality. Furthermore, if a robot is disconnected from the rest, then it can take a greedy approach as described in Reference~\cite{tokekar2014multi} before they reach any other robots to communicate.

The algorithm also allows for the number of robots and targets to change over time. Since each robot determines its neighbors at each time step, any new robots or targets will be identified and become part of the time-varying local layered graphs. The robots can also be anonymous (as long as they can break the symmetry to determine which robot, amongst a set, will realize the target node, when multiple robots can observe the same target).

The number of layers, $h$, directly affects the solution quality and can be set by the user. Increasing $h$ results in better solutions at the expense of more communication. $h=0$ is equivalent to the greedy approach where no robots communicate with each other. 


\begin{table}[h]
\centering
\begin{center}
\begin{tabular}{ c c c c c } 
 \hline
 \hline \\[-1em]
 $\textbf{p}_m^i $ & $x_{m}^{i}$ & $h = 2$\ & $h = 10$ &
$h = 30$ \\ \\[-1em]
 \hline
 \hline \\[-1em]
 $\textbf{p}_1^1$ & $x_1^1 = $ & 0.5000 & 0.5000 & 0.5000 \\ \\ \\[-2em]

 $\textbf{p}_2^1$ & $x_2^1 = $ & 0.5000 & 0.5000 & 0.5000 \\ \\[-1em]
 \hline \\[-1em]
 $\textbf{p}_3^2$ & $x_3^2 = $ & 0.6667 & 0.7591 & 0.7855 \\ \\ \\[-2em]

 $\textbf{p}_4^2$ & $x_4^2 = $ & 0.3333 & 0.2409 & 0.2145 \\ \\[-1em]
 \hline \\[-1em]
 $\textbf{p}_5^3$ & $x_5^3 = $ & 0.3333 & 0.2409 & 0.2145 \\ \\ \\[-2em]

 $\textbf{p}_6^3$ & $x_6^3 = $ & 0.6667 & 0.7591 & 0.7855 \\ \\[-1em]
 \hline
 \hline
\end{tabular}
\end{center}
\caption{Solution returned by the local algorithm for the example shown in Figure~\ref{fig:general}, with all edges' weights set to $1$, as a function of $h$.}
\label{table:1}
\end{table}

The above table shows the result of applying the local algorithm to the graph in Figure~\ref{fig:general} when all edge weights were set to $1$. Three different values for $h$ were tested: $2$, $10$, and $30$. In all cases, $\textbf{p}_3^2$ and $\textbf{p}_6^3$ have larger values of $x_p$ than other nodes. Thus, the robot 2 ($\textbf{r}_2$) and the robot 3 ($\textbf{r}_3$) will select the motion primitive $3$ ($\textbf{p}_3^2$) and the motion primitive $6$ ($\textbf{p}_6^3$), respectively, after employing a rounding technique to $x_p$'s.

As the number of layers increases, the more distinct the $x_p^i$ values returned by the algorithm. Another interesting observation is that robot $1$ has the same equal value on both motion primitives of its own no matter how many number of layers are used. This is because all the targets are already observed by robots $2$ and $3$ with higher values.

\subsection{Greedy Algorithm} \label{subsec:greedy}

The greedy algorithm requires a specific ordering of the robots given in advance. The first robot greedily chooses a motion primitive that can maximize the number of targets being observed. Those observed targets are removed from the consideration. Then, the second robot makes its choice; this repeats for the rest of robots. If the communication graph is disconnected and forms more than one connected component, the greedy algorithm can independently be applied to each connected component without modifying the algorithm.
Note again that the greedy algorithm is for the \WinnerTakesAll{} version of SATA. 

\begin{algorithm} \label{alg:greedy}
    \SetKwInOut{Input}{Input}
    \SetKwInOut{Output}{Output}
    \Input{Order of robots $R$.}
    
    Initialize $w(\textbf{t}_j)=0\ \forall j\in T$.

\For{$i\in R$}
{
\For{$m\in P^i$}
{
Compute $c_{i,m}^j$$\ \forall j\in T$.

$w^\prime(\textbf{p}_{m}^i)=\sum_{j}\max\{w(\textbf{t}_{j}),c_{i,m}^j\}$.
}

Determine $x_m^i=\argmax w^\prime(\textbf{p}_m^i)\ \forall m\in P^i$.

Update $w(\textbf{t}_j)=\max\{w(\textbf{t}_j),c_{i,m}^j\}\ \forall j\in T$.

}
$y_i^j\leftarrow0\ \forall i\in R,\ j\in T$.

\For{$j\in T$}
{
$\textbf{r}_i^*\leftarrow\argmax_{i\in R}\sum_{m}c_{i,m}^jx_m^i$.

$y_{i^*}^j\leftarrow 1$.
}
    \caption{Greedy algorithm}
\end{algorithm}

As shown in Algorithm~\ref{alg:greedy}, the greedy algorithm runs in $|R|$ communication rounds at each time step. We define two additional functions: $w(\textbf{t}_j)$ gives a quality of tracking for $j$-th target; and $w^\prime(\textbf{p}_m^i)$ gives the sum of quality of tracking over all feasible targets using $m$-th motion primitive of $i$-th robot. If, for example, $c_{i,m}^j$ is used as a distance metric, the $\max$ ensures that the quality of tracking for $j$-th target is only given by the distance of the nearest robot/primitive. That is, even if multiple primitives can track the same target $j$, when counting the quality we only care about the closest one. The total quality will then be the sum of qualities for each target.

The objective in Line 5 in Algorithm~\ref{alg:greedy} appears, at first sight, to be different than that given in  Equation (\ref{eqn:objective}). The following lemma, however, proves that the two objectives are equivalent.

\begin{lemma}~\label{lemma:greedy}
Greedy algorithm of Algorithm~\ref{alg:greedy} gives a feasible solution for the \WinnerTakesAll{} version of SA- TA.
\end{lemma}
The proof is given in Appendix~\ref{append:greedy}. Since the objective in Line 5 in Algorithm~\ref{alg:greedy} is submodular, the resulting algorithm yields a $2$--approximation to \WinnerTakesAllFirst{}- \WinnerTakesAllSecond{} \cite{nemhauser1978analysis}.

The greedy algorithm can perform arbitrarily worse than the optimal solution if it is applied to the \BottleneckFirst{}- \BottleneckSecond{} version of the problem. In Appendix~\ref{append:greedy_bottleneck}, we show an example where the greedy yields an arbitrarily bad solution for the \Bottleneck{} version.

A centralized-equivalent approach is one where the robots all broadcast their local information until some robot has received information from all others. This robot can obtain a centralized solution to the problem. A centralized-equivalent approach for a complete $\mathcal{G}_C$ runs in 2 communication rounds for receiving and sending data to neighbors. However, the greedy algorithm and local algorithm have $|R|$ and $h\log(1/\epsilon)$ communication rounds, respectively, for a complete $\mathcal{G}_C$. Note that $h\ll|R|$ for most practical cases. 


\section{Simulations} \label{sec:sim}

We carried out four types of simulations to verify the efficacy of the proposed algorithms under the condition that the amount of time required for communication is limited. First, we compare the performance of the greedy and local algorithms with centralized, optimal solutions. Second, we study the effect of varying the parameters (\ie, the number of layers) for the local algorithm. Third, we describe how to implement the algorithms for sequential planning over multiple horizons and evaluate their performance over time. Last, we compare the greedy algorithm with a state-of-the-art distributed tracking algorithm.




\subsection{Comparisons with Centralized Solutions}

We performed comparison studies to verify the performance of the proposed algorithms. We compared the greedy solution with the optimal, centralized QMILP solution as well as a random algorithm as a baseline for the \WinnerTakesAll~version. We compared the local algorithm's solution with the optimal ILP solution as well as the LP with rounding for \Bottleneck{}. For these comparisons, we assumed that there are only two primitives to choose from (making the random algorithm a powerful baseline). We later analyzed the algorithms with more primitives.
We used TOMLAB~\cite{QMILP} to solve the QMILP and ILP problems. The toolbox works with MATLAB and uses IBM's CPLEX optimizer in the background. On a laptop with processor configuration of Intel® Core™ i7-5500U CPU @ 2.40GHz~x~4 and 16 GB of memory the maximum time to solve was around $3$ seconds on a case with 150 targets. Most of our cases were solved in less than $2$ seconds.

We randomly generated graphs similar to Figure~\ref{fig:general} for the comparison. To control the topology of the randomly generated graphs, we defined $\phi:\mathcal{G}_S\rightarrow\mathbb{R}$ to be the percentage of targets that are detected by a motion primitive. We denote the average degree of edges by $d_{avg}(\cdot)$. Therefore:
\begin{equation}
\phi(\mathcal{G}_S):=\frac{d_{avg}(T)}{\sum_{i=1}^{|R|}|P^i|}\times 100=\frac{|E_S|}{\sum_{i=1}^{|R|}|P^i||T|}\times 100.
\label{eqn:avg_target}
\end{equation}

We started with the upper half of the graph, connecting each robot to its two motion primitives. Then, we iterated through each of motion primitive and randomly chose a target node to create an edge. Next, we iterated through target nodes and randomly chose a motion primitive to create an edge. We also added random edges to connect disconnected components (to keep the implementation simpler). We repeated this in order to get the required graph. If we needed to increase the degree of target nodes in the graph, we created new edges to random primitives till we achieved the desired $\phi(\mathcal{G}_S)$. We generated cases by varying $\phi(\mathcal{G}_S)$, number of targets, and number of robots using the method described above. Here, the tracking quality was defined as the number of targets, \ie, $c_{i,m}^j\in\{0,1\}$ for all cases.

\begin{figure*}
\centering
\includegraphics[width=\textwidth]{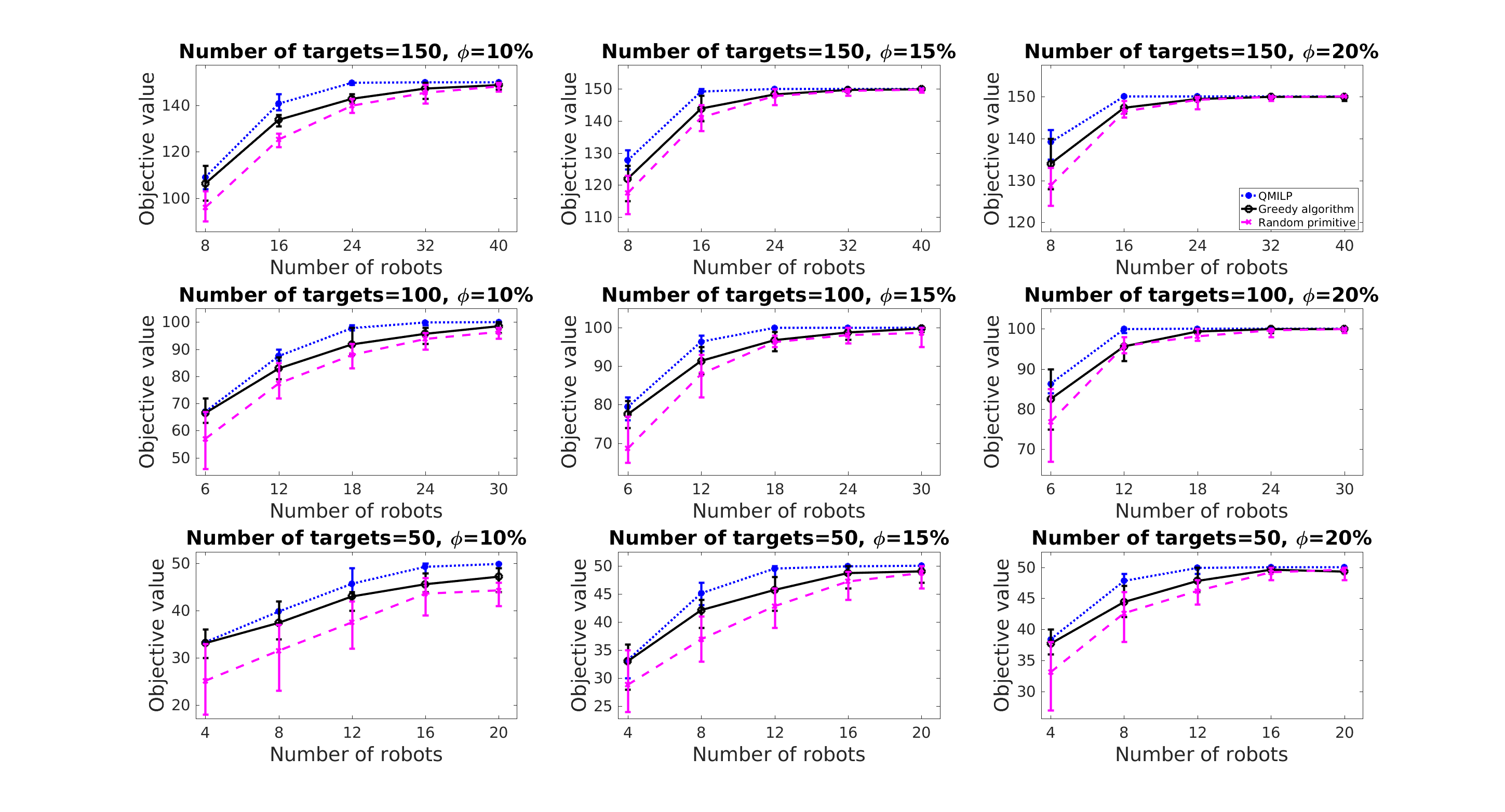}
\caption{Showing the comparative results of QMILP, greedy algorithm, and randomly choosing a motion primitive for \WinnerTakesAll. To generate the graphs, we varied number of robots, total number of targets, and $\phi(\mathcal{G}_S)$. We ran 100 trials for each case.\label{fig:cmp_plot}}
\end{figure*}

The comparative simulation results for \WinnerTakesAllFirst{}- \WinnerTakesAllSecond{} are presented in Figure~\ref{fig:cmp_plot}. The plots show minimum, maximum, and average of the targets covered by the greedy algorithm and QMILP running 100 random instances for every setting of the parameters. We also show the number of targets covered when choosing motion primitives randomly as a baseline. We observe that the greedy algorithm performs comparatively to the optimal algorithm, and is always better than the baseline. In all the figures, $\Delta_R=2$, making random a relatively powerful baseline. The difference between the greedy algorithm and the baseline becomes smaller as $\phi(\mathcal{G}_S)$ increases. This could be because of the fact that the baseline saturates at the maximum objective value with fewer robots as $\phi(\mathcal{G}_S)$ increases. As $\phi(\mathcal{G}_S)$, number of targets, and number of robots increase, the performance of the greedy algorithm also improves.


\begin{figure*}
\centering
\includegraphics[width=\textwidth]{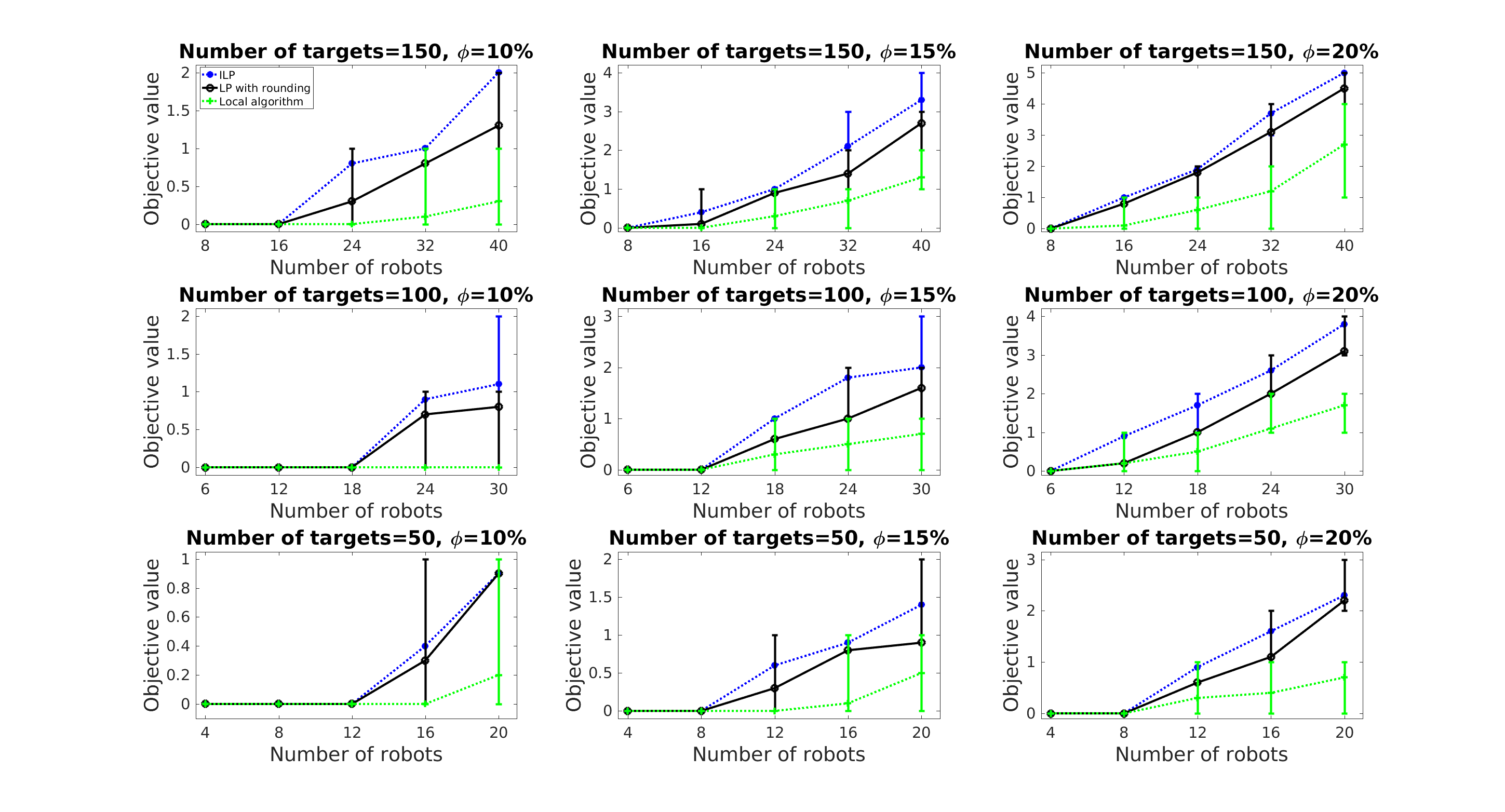}
\caption{Comparison simulation for the \Bottleneck~version of the ILP, LP with rounding, local algorithm and randomly choosing a motion primitive. We set $h$ to 2 in the local algorithm, for all cases. Each case was obtained from 100 trials.
}\label{fig:cmp_plot_bottleneck}
\end{figure*}

Figure~\ref{fig:cmp_plot_bottleneck} shows the comparison results for \BottleneckFirst{}- \BottleneckSecond{} where the objective values were computed from the $w$ term of Equation (\ref{eqn:mpcp}). As the proposed local algorithm is a linear relaxation of the ILP formulation, we compared the local solution with the optimal ILP solution. Note that both the ILP and LP with rounding are centralized methods. If the solution value is $0$, this means that at least one target is not covered by any selected motion primitives. The specific configuration of input motion primitives and target states is such that no matter what motion primitives are chosen, at least one target will be left uncovered. This means that the bottleneck objective (\ie, the optimal value of ILP) is $0$. 
If the mean value is larger than $0$, this implies that all targets are covered by at least one motion primitive on average. The ILP and LP with rounding outperform the local algorithm in all cases. Nevertheless, we find that the local algorithm performs comparably to the centralized methods (and far better than the theoretical bound).

\subsection{Effect of $h$ for the Local Algorithm}

\begin{figure}[htb]
\centering
\subfigure[LP with rounding.]{\includegraphics[width=0.49\columnwidth]{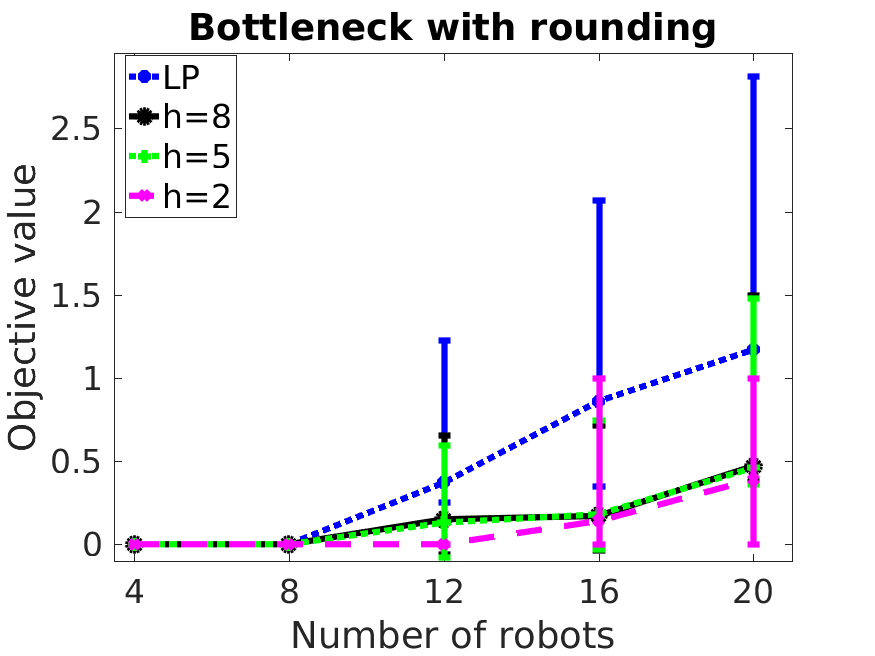}}
\subfigure[LP without rounding.]{\includegraphics[width=0.49\columnwidth]{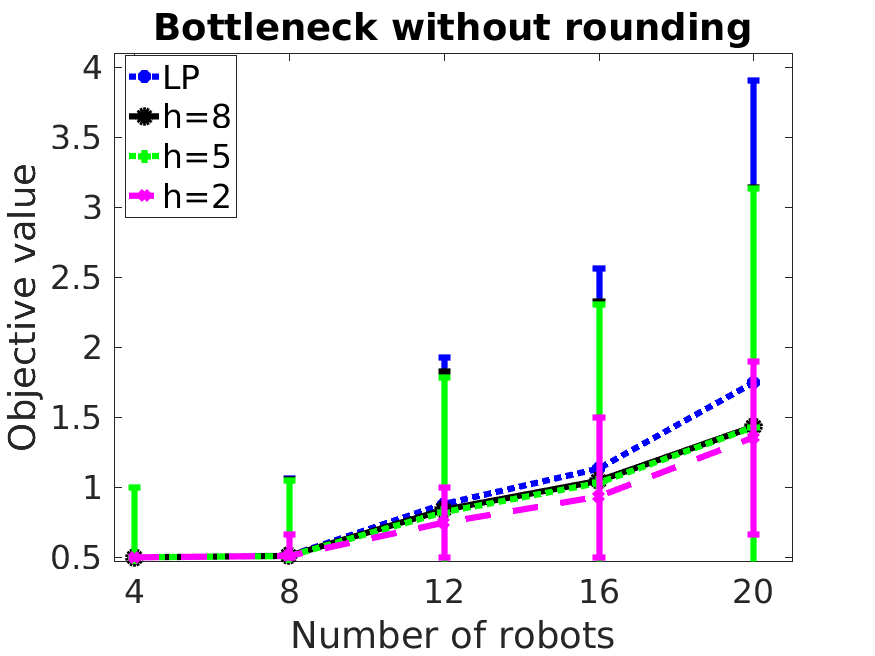}}
\caption{Analysis of varying the number of layers ($h$) for the local algorithm. The number of targets used is $50$ and $\phi(\mathcal{G}_S)=15\%$. We ran 100 trials for each case.
}
\label{fig:diff_h}
\end{figure}

We analyzed the performance of the local algorithm for different number of layers (\ie, $h$), as shown in Figure~\ref{fig:diff_h}. The LP value (without rounding) is the upper bound on the optimal solution. We observed how much the rounding sacrifices by comparing the LP with and without the rounding. In the case where $h$ was set to $5$ and $8$ for both with and without the rounding, there is no evident difference between them. This implies that $h$ should not necessarily be large as it does not contribute to the solution quality much (as also seen in Theorem~\ref{theorem:proposed_local}). In other words, the local algorithm does not require a large number of communication hops among robots, which is a powerful feature of the local algorithm.

\subsection{Multi-robot Multi-target Tracking over Time}

The greedy and local algorithms find the motion primitives to be applied over a single horizon. In order to track over time, the SATA problem will need to be solved repeatedly at each time step. In this section, we describe how to address this and other challenges associated with a practical implementation. We demonstrate a realistic scenario of cooperative multi-target tracking in the Gazebo simulator using ROS (Figure~\ref{fig:gazebo_greedy}). A bounded environment consists of dynamic targets that move in a straight line and change their heading direction randomly after a certain period. The motion model is not known to the robots.

\begin{figure}[thpb]
\centering
\includegraphics[scale=0.28]{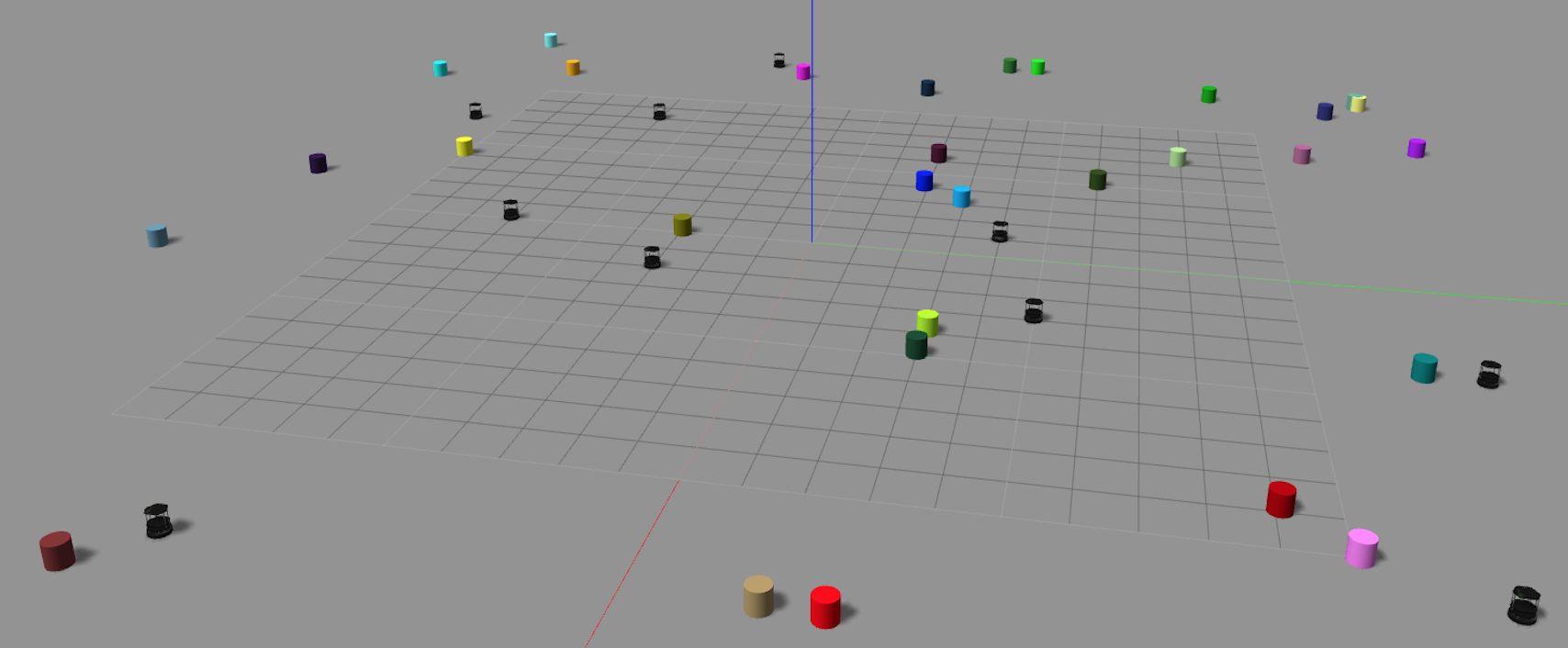}
\caption{Gazebo simulator showing ten robots tracking thirty randomly moving targets. We set the sensing and communication ranges to $5m$ and $10m$, respectively.
} 
\label{fig:gazebo_greedy}
\end{figure}

\paragraph{Greedy Algorithm.}
We implemented the greedy algorithm to solve the \WinnerTakesAll{} variant in a fully distributed fashion. There was no centralized algorithm and each robot was a separate ROS node that only had access to the local information. Each robot had its local estimator that estimated the state of targets within its FOV. We simulated proximity-limited communication range such that only robots that can see the same target can exchange messages with each other.

A sketch for the implementation of the greedy algorithm is as follows. Each robot has a local timer which is synchronized with the others. Each robot also knows its own ID which is also the order in which the sequential decisions are made. We partition the planning horizon into two periods. In the first \emph{selection} period, the robots choose their own primitives sequentially using the greedy algorithm. In the second \emph{execution} period, the robots apply their motion primitives and obtain measurements of the target.

In the selection period, a robot waits for the predecessor robots (of lower IDs) to make their selections. Every robot knows when it is its turn to select a motion primitive (since the order is fixed). Before its turn, a robot simply keeps track of the most recent $w(\textbf{t}_j)$ vector received from a predecessor robot within communication range. During its turn, the robot chooses its motion primitive using the greedy algorithm, and updates the $w(\textbf{t}_j)$ vector based on its choice. It then broadcasts this updated vector to the neighbors, and waits for the selection period to end. Then, each robot applies its selected motion primitive till the end of the horizon. The process repeats after each planning horizon. The selection period can be preceded by a sensor fusion period, where the robots can execute, for example, the covariance intersection algorithm~\cite{niehsen2002information}.

For simulations we set the selection and execution periods times to $0.2|R|s$ and $6s$, respectively, where $|R|$ is the number of robots. Each robot made its choice after $0.2s$ within the selection period. Each robot had a precomputed library of 21 motion primitives including staying in place. It should be noted that our algorithms do not require a motion primitive of stay in place. Each robot had a disk-shaped FOV. The sensing and communication ranges were set to $5m$ and $10m$, respectively. We tested both the inverse of the distance and the number of targets as tracking quality (which defines $c_{i,m}^j$).

We carried out simulations using ten robots tracking thirty moving targets, as shown in Figure~\ref{fig:gazebo_greedy}. Initial positions of robots and targets were randomly chosen in a $30\times 30m$ square environment. It may be possible that some targets were outside the FOV of any robots in the beginning.

\begin{figure}[thpb]
\centering
\includegraphics[scale=0.60]{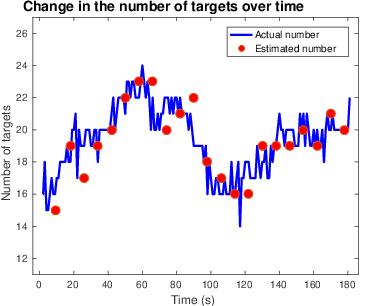}
\caption{Change in the number of targets over time when ten robots are tracking thirty moving targets.} 
\label{fig:change_num}
\end{figure}

\begin{figure}[htb]
\centering
\subfigure[Actual number of targets.]{\includegraphics[width=0.49\columnwidth]{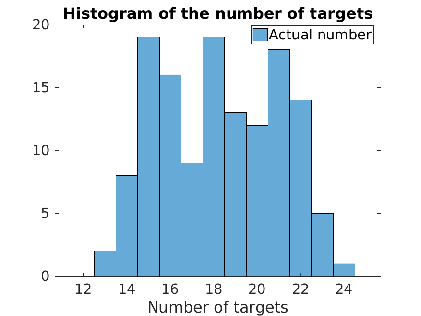}}
\subfigure[Estimated number of targets.]{\includegraphics[width=0.49\columnwidth]{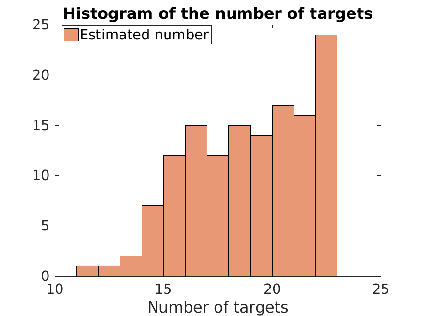}}
\caption{Histogram of the number of targets. \label{fig:histo_num}}
\end{figure}

Figure~\ref{fig:change_num} shows the change in the number of targets over time from a randomly generated instance where the objective was to track the most number of targets. We show both the estimated number of targets and the actual number of targets. The estimated number is the value of the solution found at the end of the selection period (obtained every $8s$). This is based on the predicted trajectory of the targets.\footnote{Although we model linear motion for the targets, more sophisticated models for the prediction of target states can also be employed.} The actual number of targets was found by counting the target that is within the FOV of any robots during the execution period. Figure~\ref{fig:histo_num} shows the histogram of the actual and estimated number of targets for 10 trials, each lasting three minutes.

\begin{figure}[thpb]
\centering
\includegraphics[scale=0.60]{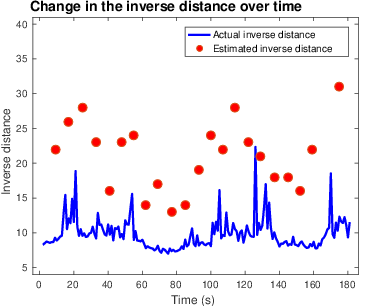}
\caption{Change in the inverse of the distance over time when ten robots are tracking thirty moving targets.} 
\label{fig:change_dist}
\end{figure}

\begin{figure}[htb]
\centering
\subfigure[Actual inverse distance.]{\includegraphics[width=0.49\columnwidth]{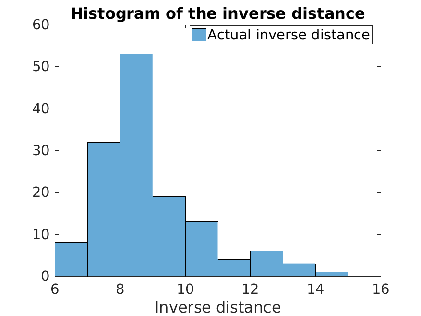}}
\subfigure[Estimated inverse distance.]{\includegraphics[width=0.49\columnwidth]{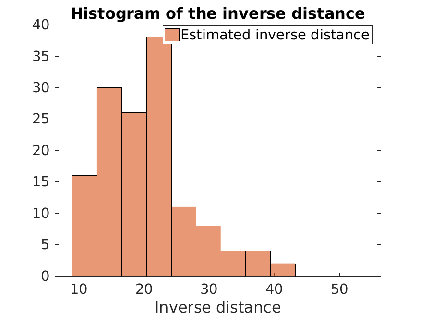}}
\caption{Histogram of the inverse distance. \label{fig:histo_dist}}
\end{figure}

Figures~\ref{fig:change_dist} and \ref{fig:histo_dist} show the corresponding plots when the objective was to maximize the total quality of tracking (inverse distance to the targets). Here, we saw that the estimated and the actual values differed much more than the previous case. We conjectured that this was due to the fact that the uncertainty in the motion model of the robots, targets, and measurements had a larger effect on the actual quality of tracking as compared to the number of targets tracked. For instance, even if the actual state of the target deviates from the predicted state, it is still likely that the target will be in the FOV. However, the actual distance between the robot and the target may be much larger than estimated.

\begin{figure}[thpb]
\centering
\includegraphics[scale=0.34]{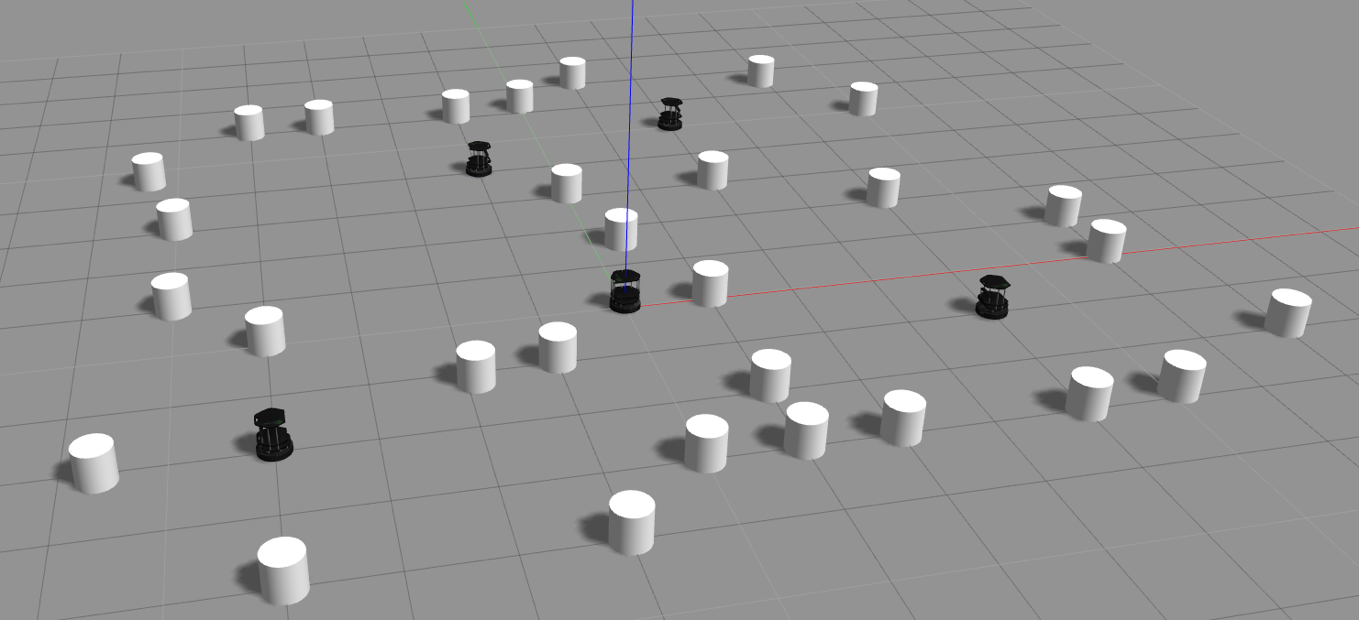}
\caption{Snapshot of the Gazebo simulator that shows when five robots are tracking thirty stationary and moving targets. The sensing and communication ranges were set to 3$m$ and $6m$, respectively.
} 
\label{fig:gazebo_local}
\end{figure}

\paragraph{Local Algorithm.} We also implemented the proposed local algorithm as shown in Figure~\ref{fig:gazebo_local}. Five mobile robots were deployed to track thirty targets (a subset of which were mobile) with a FOV of 3$m$ on the xy plane. For each robot two motion primitives were used: one was to remain in the same position and the other one was randomly generated between $-30^{\circ}$ and $30^{\circ}$ of the robot's heading traveling randomly up to $1m$. 

\begin{figure}[thpb]
\centering
\includegraphics[scale=0.60]{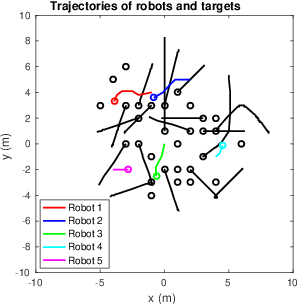}
\caption{Plot of trajectories of robots and targets applying the local algorithm to the simulation given in Figure~\ref{fig:gazebo_local}. Black lines represent trajectories of thirty targets. $\circ$ denotes the end position of trajectories. The algorithm was performed for 40 seconds.}
\label{fig:local_trajectory}
\end{figure}

\begin{figure}[htb]
\centering
\subfigure[Total number of targets observed.]{\includegraphics[width=0.49\columnwidth]{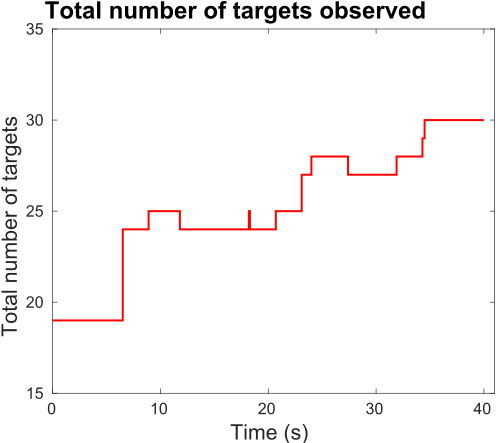}}
\subfigure[Average number of targets observed over time.]{\includegraphics[width=0.49\columnwidth]{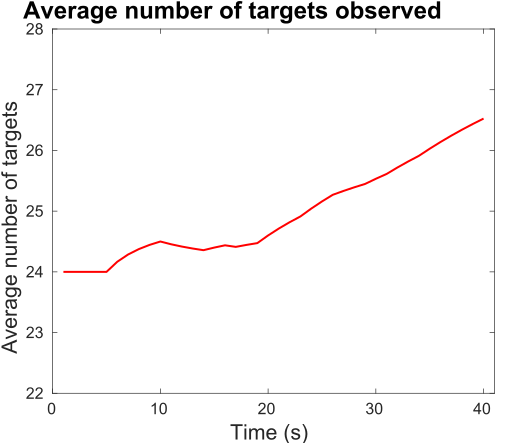}}
\caption{Change in the total and average number of targets being observed by any robots over time.}
\label{fig:total_num_local}
\end{figure}

The objective of this simulation is to show the performance of the proposed algorithm for the \Bottleneck{} version. At each time step, the local algorithm was employed to choose motion primitives that maximize the minimum number of targets being observed by any robots.


Figure~\ref{fig:local_trajectory} shows the resultant trajectories of robots and targets obtained from the simulation. Figure~\ref{fig:total_num_local} presents the (total/average) number of targets tracked by the local algorithm for a specific instance. Although the local algorithm has a sub-optimal performance guarantee, we observe that in practice, it performs comparably to the optimal path.

\subsection{Comparison of the Greedy Algorithm with Other CMOMMT Algorithm}

We compared the greedy algorithm with an algorithm proposed by Parker~\cite{parker2002distributed} following the CMOMMT approach. This algorithm addresses the same objective as the \WinnerTakesAll. Parker's algorithm computes a local force vector for all robots (attraction by nearby targets and repulsion by other nearby robots). It does not require any central unit to determine their motion primitives and considers limited sensing and communication ranges, similar to this paper. Parker's algorithm determines the moving direction of robots by using the local force vector and moves the robots along this direction until they meet the available maneuverability at each time step. However, no theoretic guarantee with respect to the optimal solution was provided by this algorithm.

We created an environment of $200\times 200m$ square for comparison using MATLAB. The robots can move $10m$ per time step while the targets can move $5m$ per time step and randomly changed their direction every $25$ time steps. If the targets met the boundary of the environment, they picked a random direction that kept them within the environment. In each instance, robots and targets were randomly located initially. The sensing and communication ranges were set to $40m$ and $80m$, respectively. 

We empirically studied two cases: the first is to evaluate the objective value of the proposed greedy algorithm and Parker's algorithm for the same problem instance at a given time step; and the second is to apply the two algorithms over $200$ time steps starting from the same configuration.

\begin{figure*}[h]
\centering
\subfigure[]{\includegraphics[width=0.68\columnwidth]{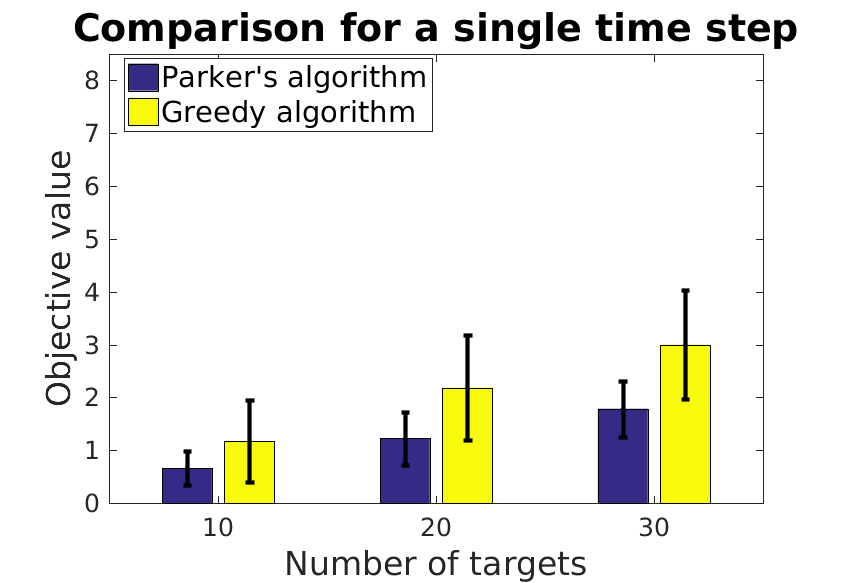}}
\subfigure[]{\includegraphics[width=0.68\columnwidth]{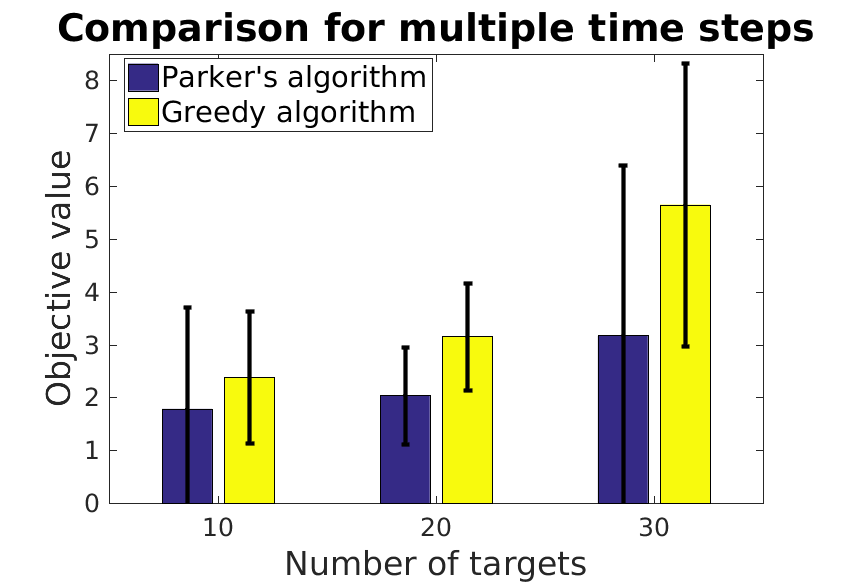}}
\subfigure[]{\includegraphics[width=0.68\columnwidth]{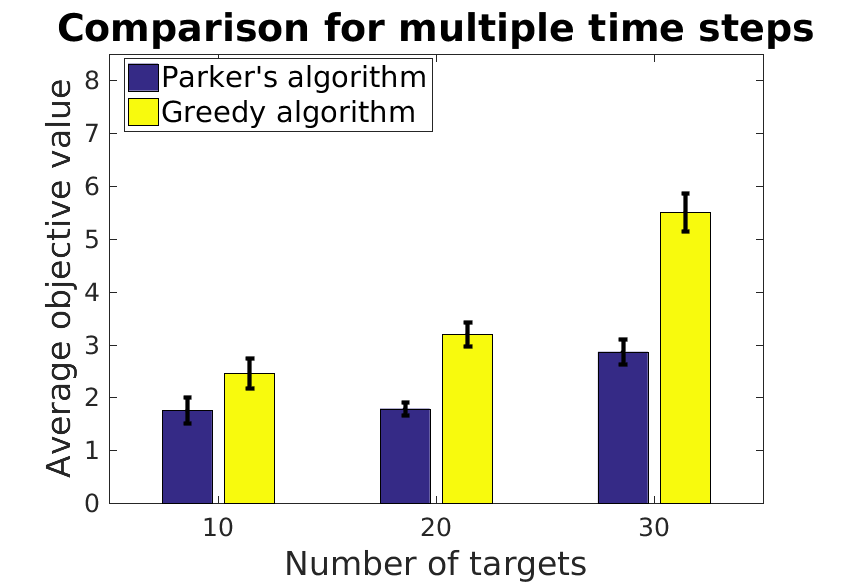}}
\caption{Comparison with the Parker's algorithm~\cite{parker2002distributed}. (a) $200$ instances were run. (b) $200$ time steps were run. (c) $200$ time steps were run to compare the metric proposed by Parker~\cite{parker2002distributed}. We used $10$ robots for all cases. We ran $10$ trials for (b) and (c). Bar graphs show the mean and standard deviation for different number of targets ($10$, $20$ and $30$ targets).\label{fig:comparative}
}
\end{figure*}

When both algorithms were applied to the same problem setup (Figure~\ref{fig:comparative}(a)), the objective values for both algorithms increased as the number of targets increased. Nevertheless, the greedy algorithm outperformed Parker's algorithm. This can be attributed to the fact that Parker's algorithm computes the local force vector based on a heuristic (get closer to the targets) but does not explicitly optimize the objective function of \WinnerTakesAll{}. In Figures~\ref{fig:comparative}(b) and~\ref{fig:comparative}(c), similar results can be seen when both algorithms generate different trajectories for robots after $200$ time steps. The comparison measure used in Figure~\ref{fig:comparative}(c) is the average of the objective value over time, first proposed by Parker~\cite{parker2002distributed}. These empirical simulations show the superior performance of the greedy algorithm over the existing method.

In summary, we find that our algorithms perform comparably with centralized, optimal algorithms and outperform the baseline algorithm. We also find that greedy algorithm has better performance than the decentralized algorithm from Parker~\cite{parker2002distributed}. In theory, the performance bound for the local algorithm worsens as $h$, the amount of communication available, decreases. However, in practice, we find that the local algorithm does not require a large number of layers to yield good performance, which reduces the computational and communication burden.





\section{Conclusion} \label{sec:conc}

This paper gives a new approach to solve the multi-robot multi-target assignment problem using greedy and local algorithms. Our work is motivated by scenarios where the robots would like to reduce their communication to solve the given assignment problem while at the same time maintaining some guarantees of tracking. We used powerful local communication framework employed by Flor\'een~\etal~\cite{floreen2011local} to leverage an algorithm that can trade-off the optimality with communication complexity. We empirically evaluated this algorithm and compared it with the baseline greedy strategies. 

Our immediate future work is to expand the scope of the problem to solve both versions of SATA over multiple horizons. In principle, we can replace each motion primitive input with a longer horizon trajectory and plan for multiple time steps (say, $H$ time steps). However, this comes at the expense of increased number of trajectories ${|P^{i}|}^H$ to choose from 
which will result in increased computational time. Furthermore, planning for a longer horizon will require prediction of targets' states far in the future which can lead to poorer tracking performance. We are also working on implementing the resulting algorithms on actual aerial robotic systems to carry out real-world experimentation.

\section*{Acknowledgement}

The authors would like to thank Dr. Jukka Suomela from Aalto University for fruitful discussion.

\addtolength{\textheight}{-3cm}


\bibliographystyle{IEEEtran}
\bibliography{IEEEabrv,yoon_refs}

\appendix

\section{Proof of Lemma~\ref{lemma:equiv_bottleneck}}~\label{append:equiv_bottleneck}

Equation (\ref{eqn:mpcp}) of a max-min linear program is equivalent to the following max-min problem if the scalar variable $w$ which represents the inner minimization is eliminated:
\begin{equation}
\begin{split}
\max_{x_{m}^{i}}\min_{j\in T}\ &\left(\sum_{i\in R}\sum_{m\in P^i}c_{i,m}^jx_{m}^{i}\right) \\
\mbox{subject to}\ \ \ &\sum_{m\in{P^i}}x_{m}^{i}\leq{1}\ \ \forall{i\in{R}} \\
&\ \ \ \ \ \ \ \ x_{m}^{i}\geq{0}\ \ \forall{m\in{P^i}}.
\end{split}
\label{eqn:equiv_mpcp}
\end{equation}

From Equations (\ref{eqn:mpcp}) and (\ref{eqn:equiv_mpcp}), the following relationship is satisfied:

\begin{equation}
w^*=\min_{j\in T}\left(\sum_{i\in R}\sum_{m\in P^i}c_{i,m}^j{x_{m}^{i}}^*\right).
\label{eqn:equi_relation}
\end{equation}

Since Equation (\ref{eqn:bottleneck}) does not require $x_{m}^{i}$ to be a linear value, Equation (\ref{eqn:bottleneck}) is equivalent to Equation (\ref{eqn:mpcp}) with additional integer constraints.

\section{Proof of Lemma~\ref{lemma:greedy}}~\label{append:greedy}

Considering $c_{i,m}^j$, which is a weight between $m$-th motion primitive of $i$-th robot and $j$-th target on graph $\mathcal{G}_S$, a quality of tracking ($w(\textbf{t}_j)$) for $j$-th target can be defined as follows:

\begin{equation}
w(\textbf{t}_j)\triangleq \max\{c_{i,m}^j\big|x_m^i=1,\ \forall i\in R, m\in P^i\}.
\label{eqn:w_function}
\end{equation}

Therefore, the sum of quality of tracking over all targets is:
\begin{equation}
\begin{split}
\sum_{j\in T} w(\textbf{t}_j)&=\sum_{j\in T}\max\{c_{i,m}^j\big|x_m^i=1,\ \forall i\in R, m\in P^i\} \\
&=\sum_{j\in T}\Big(\sum_{i\in R}y_{i}^{j}\Big(\sum_{m\in P^i}c_{i,m}^jx_{m}^{i}\Big)\Big).
\label{eqn:sum_w_function}
\end{split}
\end{equation}

Equation (\ref{eqn:sum_w_function}) is obtained by taking into account the conditional term of the first equation explicitly.
The last equation follows from the property that $y_i^j$ chooses the maximum value of $\sum_{m\in P^i}c_{i,m}^jx_{m}^{i}$ among all robots, which is shown in lines 10-14 of Algorithm~\ref{alg:greedy}. Therefore, the last equation is equal to the inner term of Equation (\ref{eqn:objective}).

\section{Greedy Performs Poorly for the \Bottleneck{} Variant}~\label{append:greedy_bottleneck}

We present an example of instance that shows an arbitrary poor performance of the greedy algorithm when applied to the \Bottleneck{} variant. Consider the following case where there are two robots ($\textbf{r}_i$) having two motion primitives ($\textbf{p}_m^i$) for each and two targets. The realization of the communication and sensing graphs are as in the following table. The tracking quality in this example corresponds to the number of targets being tracked.

\begin{center}
\begin{tabular}{ |c|c|c| }
 \hline
  & $\textbf{p}_1^1$, $\textbf{p}_1^2$ & $\textbf{p}_2^1$, $\textbf{p}_2^2$ \\ 
 \hline
 $\textbf{r}_1$ & $\textbf{t}_1$ & $\varnothing$ \\  
 \hline
 $\textbf{r}_2$ & $\varnothing$ & $\textbf{t}_2$ \\
 \hline
\end{tabular}
\label{table:2}
\end{center}

Let's apply the \Bottleneck{} version of greedy algorithm to this case. Since the objective of the \Bottleneck{} variant is to maximize the minimum tracking quality, the robot $1$ ($\textbf{r}_1$) chooses motion primitive $2$ ($\textbf{p}_2^1$) because choosing motion primitive $1$ ($\textbf{p}_1^1$) gives the value of $1$ while choosing motion primitive $2$ ($\textbf{p}_2^1$) gives the value of $0$. For the same reason, the robot $2$ ($\textbf{r}_2$) chooses motion primitive $1$ ($\textbf{p}_1^2$). This gives the total value of $0$, whereas the optimal solution is $2$ as the first robot and second robot choose motion primitive $1$ ($\textbf{p}_1^1$) and motion primitive $2$ ($\textbf{p}_2^2$), respectively. The similar case is reproducible with a larger number of robots, motion primitives, and targets. Thus, the simple greedy performs arbitrarily badly for the \Bottleneck{} variant.


\vfill

\end{document}